\documentclass[10pt,twocolumn,letterpaper]{article}

\usepackage{cvpr}
\usepackage{times}
\usepackage{epsfig}
\usepackage{graphicx}
\usepackage{amsmath}
\usepackage{amssymb}
\usepackage{multirow}
\usepackage{tabu}
\usepackage{booktabs}
\usepackage{bbm}

\usepackage{amsthm}
\usepackage{chngcntr}

\usepackage[breaklinks=true,bookmarks=false]{hyperref}

\cvprfinalcopy 

\def\cvprPaperID{2003} 

\ifcvprfinal\fi
\begin{document}

\title{Modeling Point Clouds with Self-Attention and Gumbel Subset Sampling}

\author{Jiancheng Yang\textsuperscript{1,2}\quad\quad\quad\quad Qiang Zhang\textsuperscript{1,2} \quad\quad\quad\quad Bingbing Ni\textsuperscript{1}\thanks{Corresponding author.} \\
{\tt\small \{jekyll4168, zhangqiang2016, nibingbing\}@sjtu.edu.cn}\\
Linguo Li\textsuperscript{1,2} \quad\quad\quad\quad Jinxian Liu\textsuperscript{1,2} \quad\quad\quad\quad Mengdie Zhou\textsuperscript{1,2} \quad\quad\quad\quad Qi Tian\textsuperscript{3}\\ 
{\tt\small \{LLG440982,liujinxian,dandanzmd\}@sjtu.edu.cn},\quad{\tt\small tian.qi1@huawei.com}\\
\textsuperscript{1}Shanghai Jiao Tong University\\
\textsuperscript{2}MoE Key Lab of Artificial Intelligence, AI Institute, Shanghai Jiao Tong University\\
\textsuperscript{3}Huawei Noah’s Ark Lab
}

\maketitle

\ifcvprfinal\thispagestyle{empty}\fi

\begin{abstract}

Geometric deep learning is increasingly important thanks to the popularity of 3D sensors. Inspired by the recent advances in NLP domain, the self-attention transformer is introduced to consume the point clouds. We develop Point Attention Transformers (PATs), using a parameter-efficient Group Shuffle Attention (GSA) to replace the costly Multi-Head Attention. We demonstrate its ability to process size-varying inputs, and prove its permutation equivariance.
Besides, prior work uses heuristics dependence on the input data (\eg, Furthest Point Sampling) to hierarchically select subsets of input points. Thereby, we for the first time propose an end-to-end learnable and task-agnostic sampling operation, named Gumbel Subset Sampling (GSS), to select a representative subset of input points. Equipped with Gumbel-Softmax, it produces a "soft" continuous subset in training phase, and a "hard" discrete subset in test phase. By selecting representative subsets in a hierarchical fashion, the networks learn a stronger representation of the input sets with lower computation cost. Experiments on classification and segmentation benchmarks show the effectiveness and efficiency of our methods.
Furthermore, we propose a novel application, to process event camera stream as point clouds, and achieve a state-of-the-art performance on DVS128 Gesture Dataset.
   
\end{abstract}


\section{Introduction}

\begin{figure}[ht]
\centering
\includegraphics[width=\linewidth]{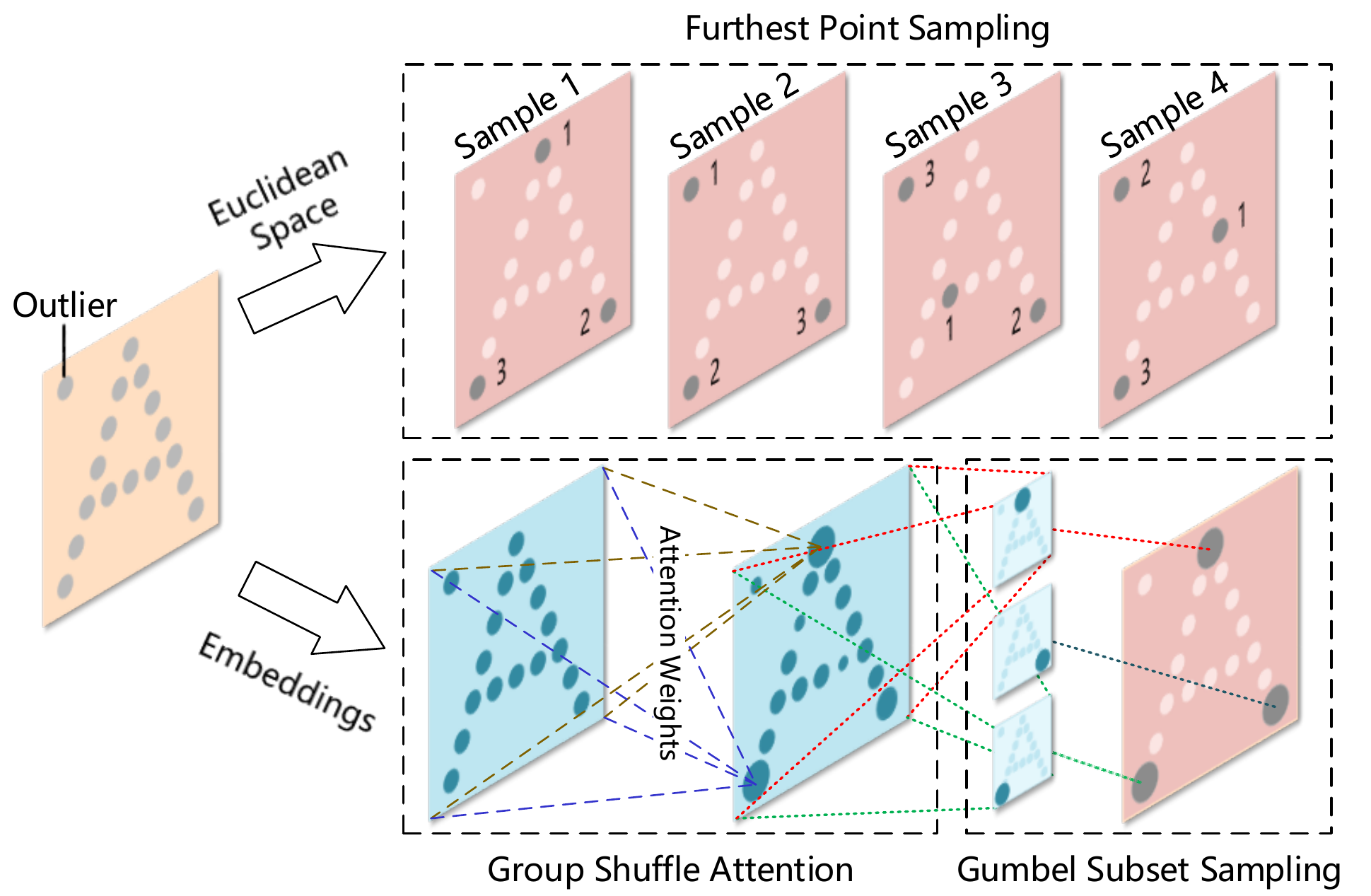}
\caption{\textbf{Illustration of Point Attention Transformers (PATs).} The core operations of PATs are {\em Group Shuffle Attention (GSA)} and {\em Gumbel Subset Sampling (GSS)}. \textit{GSA} is a parameter-efficient self-attention operation on learning relations between points. \textit{GSS} serves as a differentiable alternative to Furthest Point Sampling (FPS) in point cloud reasoning. Several flaws hinder the performance of FPS: it is dependent on the initial sampling point (\ie, permutation-variant), and it samples from low-dimension Euclidean space, making it sensitive to outliers. Instead, \textit{GSS} is permutation-invariant, task-agnostic and differentiable, which enables end-to-end learning on high-dimension representations. It is trained smoothly with Gumbel reparameterization trick, and produces hard and discrete sampling in test phase by annealing.}
\label{fig:motivation}
\end{figure}

We live in a 3D world. Geometric data have raised increasing research concerns thanks to the popularity of 3D sensors, \eg, LiDAR and RGB-D cameras. In particular, we are interested in analyzing 3D point clouds with end-to-end deep learning, which theoretically requires the neural networks to consume 1) size-varying and 2) permutation-invariant sets. PointNets \cite{qi2016pointnet} and DeepSets \cite{zaheer2017deep} pioneer directly processing the point sets. Several studies push this research direction by proposing either structural \cite{klokov2017escape, pointnetplusplus, wang2018dynamic} or componential \cite{pointcnn, shen2018mining} improvements.

We argue the relations between points are critical to represent a point cloud: a single point is non-informative without other points in the same set; in other words, it is simply represented by relations between other points. Inspired by the recent advances in NLP domain \cite{vaswani2017attention, devlin2018bert}, we introduce {\em Point Attention Transformers (PATs)}, based on self-attention to model the relations with powerful multi-head design \cite{vaswani2017attention}. Combining with ideas of the light-weight but high-performance model, we propose a parameter-efficient {\em Group Shuffle Attention (GSA)} to replace the costly Multi-Head Attention \cite{vaswani2017attention} with superior performance.

Besides, prior studies \cite{pointnetplusplus,pointcnn} demonstrate the effectiveness of hierarchical structures in point cloud reasoning. By sampling central subsets of input points and grouping them with graph-based operations at multiple levels, the hierarchical structures mimic receptive fields in CNNs with bottom-up representation learning. Despite great success, we however figure out that the sampling operation is a bottleneck of the hierarchical structures.

Few prior works study sampling from high-dimension embeddings. The most popular sampling operation on 3D point clouds is the Furthest Point Sampling (FPS). However, it is task-dependent, \ie, designed for low-dimension Euclidean space exclusively, without sufficiently utilizing the semantically high-level representations. Moreover, as illustrated in Figure \ref{fig:motivation}, FPS is permutation-variant, and sensitive to outliers in point clouds. 

To this end, we propose a task-agnostic and permutation-invariant sampling operation, named {\em Gumbel Subset Sampling (GSS)}, to address the set sampling problem. Importantly, it is end-to-end trainable. To our knowledge, we are the first study to propose a differentiable subset sampling method. Equipped with Gumbel-Softmax \cite{jang2016categorical, maddison2016concrete}, our \textit{GSS} samples soft virtual points in training phase, and produces hard selection in test phase via annealing. With \textit{GSS}, our PAT classification models are better-performed with lower computation cost.


\section{Preliminaries}

\subsection{Deep Learning on Point Clouds}
CNNs (especially 3D CNNs \cite{qi2016volumetric, zhao20183d}) dominate early-stage researches of deep learning on 3D vision, where the point clouds are rendered into 2D multi-view images \cite{su2015multi} or 3D voxels \cite{qi2016volumetric}. These methods require compute-intensively pre-rendering the sparse points into voluminous representations with quantization artifacts \cite{qi2016pointnet}. To improve memory efficiency and running speed, several researchers \cite{wang2017cnn, su2018splatnet} introduce sparse CNNs on specific data structures.

On the other hand, deep learning directly on the Euclidean-space point clouds raises research attention. By design, these networks should be able to process 1) size-varying and 2) permutation-invariant (or permutation-equivariant) point sets (called {\em Theoretical Conditions} for simplicity). PointNet \cite{qi2016pointnet} and DeepSet \cite{zaheer2017deep} pioneer this direction, where a symmetric function (\eg, shared FC before max-pooling) is used for learning each point's high-level representation before aggregation; However, relations between the points are not sufficiently captured in this way. To this end, PointNet++ \cite{pointnetplusplus} introduces a hierarchical structure based on Euclidean-space nearest-neighbor graph, Kd-Net \cite{klokov2017escape} designs spatial KD-trees for efficient information aggregation, and DGCNN \cite{wang2018dynamic} develops a graph neural network (GNN) approach with dynamic graph construction. Not all studies satisfy both {\em Theoretical Conditions} at the same time; For instance, Kd-Net \cite{klokov2017escape} resamples the input points to evade the "size-varying" condition, and PointCNN \cite{pointcnn} groups and processes the points via specific operators without "permutation-invariant" condition. 

In a word, the follow-up researches on point clouds are mainly 1) voxel-based CNNs \cite{wang2017cnn, cohen2018spherical}, 2) variants of geometric deep learning \cite{bronstein2017geometric}, \ie, graph neural networks on point sets \cite{xie2018attentional, shen2018mining, pointcnn, li2018so}, or 3) hybrid \cite{zhou2017voxelnet, le2018pointgrid, su2018splatnet}. There is also research on the adversarial security on point clouds \cite{yang2019adversarial}

\subsection{Self-Attention}
An attention mechanism \cite{bahdanau2014neural, xu2015show, gehring2017convolutional, DBLP:conf/ijcai/YanNY17, DBLP:conf/mm/YanNY17} uses input-dependent weights to linearly combine the inputs. Mathematically, given an input $X \in \mathbb{R}^{N \times c}$, a query $Q \in \mathbb{R}^{N_Q \times c}$ to attend to the input $X$, the output of the attention layer is
\begin{equation}
    Attn(Q,X) = A\cdot X=S(Q,X)\cdot X,
\end{equation}
where $S: \mathbb{R}^{N_Q \times c}\times\mathbb{R}^{N\times c} \rightarrow \mathbb{R}^{N_Q \times N}$ is a matrix function for producing the attention weights $A$. The common choices of $S$ function are additive, dot-product and general attention \cite{luong2015effective}. 
A self-attention is simply to let the inputs attend to every input element themselves, \ie, $Q=X$. As the attention layer is a single linear combination once the attention weights are produced, we call this form the {\em vanilla self-attention}.

An attention transformer \cite{vaswani2017attention} is a (fully) attentional model with {\em state-of-the-art} performance on neural machine translation and other NLP tasks. Importantly, it introduces a {\em Multi-Head Attention (\textit{MHA})} to aggregate the inputs multiple times with different linear transformations. For a self-attention version \footnote{Note the projection weights for $K, Q ,V$ of Multi-Head Attention are shared in our derivation.},
\begin{multline}
    \mathit{MHA}(X) =\\
    \mathop{concat}\{Attn(X_h,X_h)\,|\,X_h=XW_h\}_{h=1,...,H},
\end{multline}
where $H$ is the number of heads, and $W_h$ is the projection weights of head $h$.
Position-wise MLPs with non-linearity are connected to the attention layers. Equipped with different attention weights, \textit{MHA} introduces stronger capacity in a single layer than the vanilla self-attention. 

\subsection{Discrete Reparameterization}
Variational Auto-Encoders (VAEs) \cite{kingma2013auto} introduce an elegant reparameterization trick to enable continuous stochastic variables to back-propagate in neural network computation graphs. However, discrete stochastic variables are non-trivial to be reparameterized. To this regard, several stochastic gradient estimation methods are proposed, \eg, REINFORCE-based methods \cite{sutton2000policy, schulman2015gradient} and Straight-Through Estimators \cite{bengio2013estimating}. 

For a categorical distribution $Cat(s_1,...,s_M)$, where $M$ denotes the number of categories, $s_i$ ($s.t. \;\sum_{i=1}^{M}(s_i)=1$) means the probability score of category $i$, a {\bf Gumbel-Softmax} \cite{jang2016categorical, maddison2016concrete} is designed as a discrete reparameterization trick, to estimate smooth gradient with a continuous relaxation for the categorical variable. Given i.i.d Gumbel noise $\mathbf{g}=(g_1,...,g_M)$ drawn from $Gumbel(0,1)$ distribution, a soft categorical sample can be drawn (or computed) by
\begin{equation} \label{eq-gumbel-softmax}
  \mathbf{y}= \mathit{softmax}((log(\mathbf{s})+\mathbf{g})/\tau),\quad s=(s_1,...,s_M).
\end{equation}
The Eq. \ref{eq-gumbel-softmax} is referred as \textit{gumbel\_softmax} operation on $\mathbf{s}$.

Parameter $\tau>0$ is the annealing temperature, as $\tau \rightarrow 0^{+}$, $y$ degenerates into the Gumbel-Max form,
\begin{equation} \label{eq-gumbel-argmax}
  \mathbf{\hat{y}}= one\_hot\_encoding(\arg\max((log(\mathbf{s})+\mathbf{g}))),
\end{equation}
which is an unbiased sample from $Cat(s_1,...,s_M)$.

In this way, we are able to draw differentiable samples (Eq. \ref{eq-gumbel-softmax}) from the distribution $Cat(s_1,...,s_M)$ in training phase. In practice, $\tau$ starts at a high value (\eg, 1.0), and anneals to a small value (\eg, 0.1). Optimization on the Gumbel Softmax distribution could be interpreted as solving a certain entropy-regularized linear program on the probability simplex \cite{mena2018learning}. In test phase, discrete samples can be drawn with Gumbel-Max trick (Eq. \ref{eq-gumbel-argmax}).

\newtheorem{proposition}{Proposition}
\counterwithin*{proposition}{section}
\section{Point Attention Transformers}

\begin{figure*}[ht]
\centering
\includegraphics[width=\linewidth]{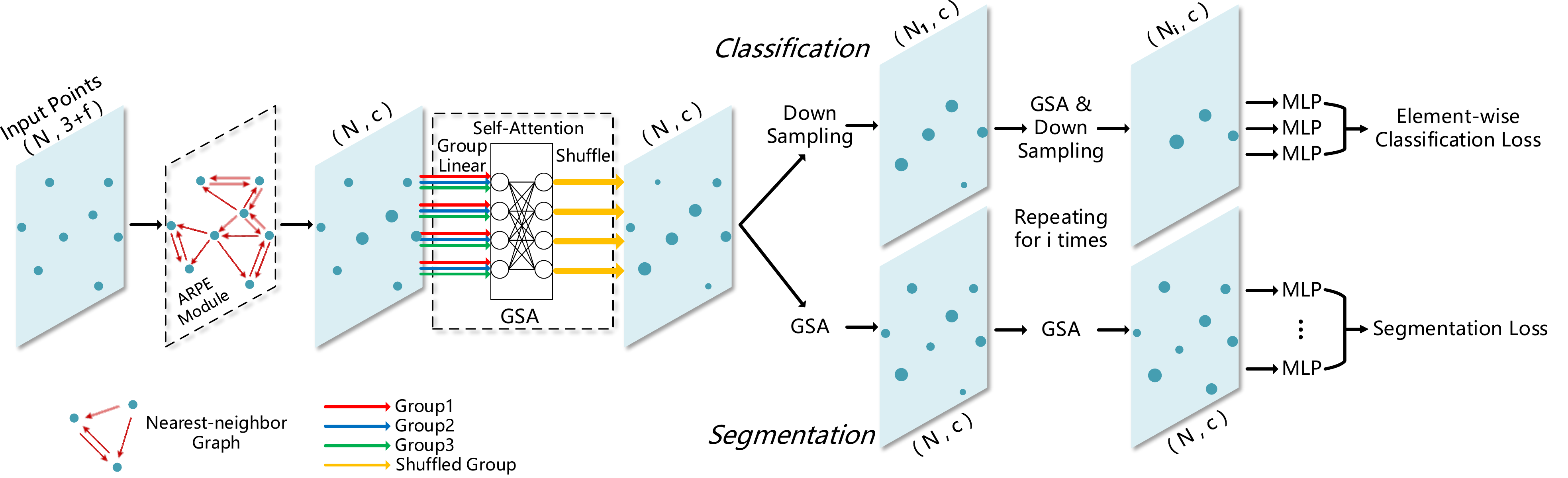}
\caption{\textbf{Point Attention Transformer architecture} for classification (top branch) and segmentation (bottom branch). The input points are first embedded into high-level representations through an {\em Absolute and Relative Position Embedding (ARPE)} module, resulting in some points representative (bigger in the figure). In classification, the features alternately pass through {\em Group Shuffle Attention (GSA)} blocks and down-sampling blocks, either Furthest Point Sampling (FPS), or our {\em Gumbel Subset Sampling (GSS)}. In segmentation, only \textit{GSA} layers are used. Finally, a shared MLP is connected to every point, followed by an element-wise classification loss or segmentation loss for training.}
\label{fig:PAT_overview}
\end{figure*}

\subsection{Overview} \label{sec:pat-overview}

We describe our model in a top-down approach. As illustrated in Figure \ref{fig:PAT_overview}, we define $N$ as the number of points, and $f$ as feature dimension except for the $xyz$-dimension, \eg, $f=3$ for RGB point clouds. An input 3D point cloud $P \in \mathbb{R}^{N\times (3+f)}$, is first embedded into higher-level representations by an {\em Absolute and Relative Position Embedding (ARPE)} module (Section \ref{sec:arpe}), in which each point is represented by its nearest neighbors' relative positions, together with its own absolute position. We then use {\em Group Shuffle Attention (GSA)} (Section \ref{sec:gsa}) blocks for mining relations between elements in the feature set $X \in \mathbb{R}^{N\times c}$, and the representation of each element becomes semantically stronger via the layer-by-layer attentional transformation.

For classification, we define $m$ as the number of target classes, the final output $y_{cls} \in \mathbb{R}^{m}$ assigns a single label to the input by $\arg\max  y_{cls}$. Inspired by several prior studies with hierarchical structures \cite{pointnetplusplus, pointcnn}, we also adopt this down-sampling structure (Section \ref{sec:down-sample}). After every \textit{GSA} operation, we sample a subset (\ie, down-sampling) for subsequent processing. The sampling operation could be either Furthest Point Sampling (FPS) or the proposed Gumbel Subset Sampling (\textit{GSS}) in Section \ref{sec:gss}. The remaining points after the last down-sampling are separately connected to shared MLPs before global average pooling for classification output. For training, a cross-entropy loss is computed over every MLP before averaging, referred as {\em Element-wise Loss} trick (Section \ref{sec:elem-loss}).

For segmentation, the output is $y_{seg} \in \mathbb{R}^{N\times m}$, which assigns a label to every point. As \textit{GSA} operation adaptively aggregates global information to every local point, the down-sampling structure is not necessary, which introduces information loss for segmentation. In this way, a segmentation PAT is simply a stack of \textit{GSA} layers connected to the \textit{ARPE} module, followed by a shared MLP on each point for pointwise segmentation.

We describe the sub-modules in the following sections.

\subsection{Absolute and Relative Position Embedding} \label{sec:arpe}

We first consider how to represent a point cloud. For a single point $x_p$, we argue that its absolute position is informative, while not rich enough; it is also represented by all the remaining points' relative positions (to $x_p$) in the same point cloud. Combine both, and we call it an {\em Absolute and Relative Position Embedding (ARPE)} module.

Given an input point cloud $X=\{x_1, x_2,...,x_p,...,x_N\}$, for a point $x_p$, its {\em position set} is defined as,
\begin{equation} \label{eq:positon-set}
    X'_p=\{(x_p,x_i-x_p)\,|\, i \neq p\}.
\end{equation}
A shared PointNet \cite{qi2016pointnet} is applied on the position set for each point, \ie,
\begin{equation} \label{eq:arpe-for-one}
    \mathit{ARPE}(x_p)=\gamma \circ \max \{h(x')\,|\,x' \in X'_p\},
\end{equation}
where $\gamma$ and $h$ are both MLPs with group normalization $\mathcal{GN}$ \cite{wu2018group} and ELU activation \cite{clevert2015fast}. Note $\mathit{ARPE}$ on all points is easy to parallelize.

With $O(N^2)$ complexity, it is too costly to use the position set $X'_p$ with all points in Eq. \ref{eq:arpe-for-one}. Instead, only top $K$ nearest neighbors are considered ("Nearest-neighbor Graph" in Figure \ref{fig:PAT_overview}). However, sparsity and number of input points are coupled; in other words, top 32 neighbors in 256 points and those in 1,024 points are very different on the scale. To make the \textit{ARPE} module more robust with various point numbers, we introduce a dilated sampling technique \cite{pointcnn}, \ie, the position set is constructed by sampling $K$ points from the top $\lfloor K\times d\rfloor $ neighbors, where dilated rate $d=d_0\cdot\frac{N}{N_0}$, and $d_0 $ is a base dilated rate on $N_0$ points. If not specified, $K=32$ and $d_0=2$ for $N_0=1,024$ points.

\subsection{Group Shuffle Attention} \label{sec:gsa}
We propose to use attention layers to capture the relations between the points. \textit{MHA} is successful in modeling relations by introducing a critical multi-head design \cite{devlin2018bert, velickovic2017graph, parmar2018image}, however we argue that it is voluminous for modeling point clouds. To this regard, we propose a parameter-efficient {\em Group Shuffle Attention (GSA)} to replace \textit{MHA}. There are two improvements over \textit{MHA}: 

Firstly, to get rid of position-wise MLPs, we integrate the non-linearity $\sigma$ into attention modules, named {\em non-linear self-attention},

\begin{equation} \label{eq:non-linear-attn}
    Attn_\sigma(Q,X) = S(Q,X)\cdot \sigma(X),
\end{equation}
where we use a Scaled Dot-Product attention \cite{vaswani2017attention} for $S$, \ie, $S(Q,X)=\mathit{softmax}(QX^T / \sqrt{c})$, and ELU activation \cite{clevert2015fast} for $\sigma$. In other words, we use the "pre-activation" to attend to the "post-activation".

Secondly, we introduce compact group linear transformations \cite{xie2017aggregated, chollet2017xception} with channel shuffle \cite{DBLP:journals/corr/ZhangZLS17, zhang2017interleaved}, keeping the multi-head design. Let $g$ be the number of groups, $c_g = c/g$, $s.t. \; c\mod g =0$, we split $X$ by channels into $g$ groups: $\{X^{(i)} \in \mathbb{R}^{N \times c_g}\}$, and define $W_i\in \mathbb{R}^{c_g \times c_g}$ as a learnable transformation weight for group $i$, thus a Group Attention (\textit{GroupAttn}) is defined,
\begin{multline}\label{eq:group-attn}
    \mathit{GroupAttn}(X) = \\
    \mathop{concat}\{Attn_\sigma(X_i,X_i) \,|\,X_i=X^{(i)}W_i\}_{i=1,...,g}.
\end{multline}
However, a pure stack of \textit{GroupAttn} blocks the information flow between groups. To enable efficient layer-by-layer transformations, we introduce a parameter-free channel shuffle \cite{DBLP:journals/corr/ZhangZLS17} operator $\psi$, see Figure \ref{fig:gsa-gss} (a) for illustration.

For an element $x \in \mathbb{R}^{c}$, we rewrite $x$ as,
\begin{equation}
\begin{split}
    x &= \{x_1,x_2,...,x_{c_g},x_{c_g+1},x_{c_g+2},...,x_c\} \\
    &=\{(x_{ic_g+j} \,|\, j=1,..,c_g) \,|\, i=0,...,g-1\},
\end{split}    
\end{equation}
where $(x_{ic_g+j} \,|\, j=1,..,c_g)$ is the $(i+1)_{th}$ group of channels. In this way, we define the channel shuffle $\psi$ as,
\begin{equation}
\begin{split}
    \psi(x) &= \{x_1,x_{c_g+1}...,x_{(g-1)c_g+1},x_2,x_{c_g+2},...,x_c\} \\
    &=\{(x_{ic_g+j} \,|\, i=0,...,g-1) \,|\, j=1,..,c_g\}.
\end{split} 
\end{equation}
For any modern deep learning framework, channel shuffle can be elegantly implemented by "reshape - transpose - flatten" end-to-end.

A Group Shuffle Attention (\textit{GSA}) is simply a Group Attention followed by the channel shuffle, together with residual connection \cite{he2016deep} and the group normalization $\mathcal{GN}$ \cite{wu2018group},
\begin{equation} \label{eq:gsa}
    \mathit{GSA}(X) = \mathcal{GN}(\psi(\mathit{GroupAttn}(X)) + X).
\end{equation}

The following proposition theoretically guarantees the permutation-equivariance of \textit{GSA}. 

\begin{proposition}
    The Group Shuffle Attention operation is permutation-equivariant, \ie, given input $X \in \mathbb{R}^{N\times c}$, $\forall$ permutation matrix $P$ of size $N$,
    \begin{equation*}
        \mathit{GSA}(P\cdot X) = P\cdot \mathit{GSA}(X).
    \end{equation*}
\end{proposition}
Proof is provided in Appendix \ref{sec:proof1}.

\begin{figure}[ht]
\centering
\includegraphics[width=\linewidth]{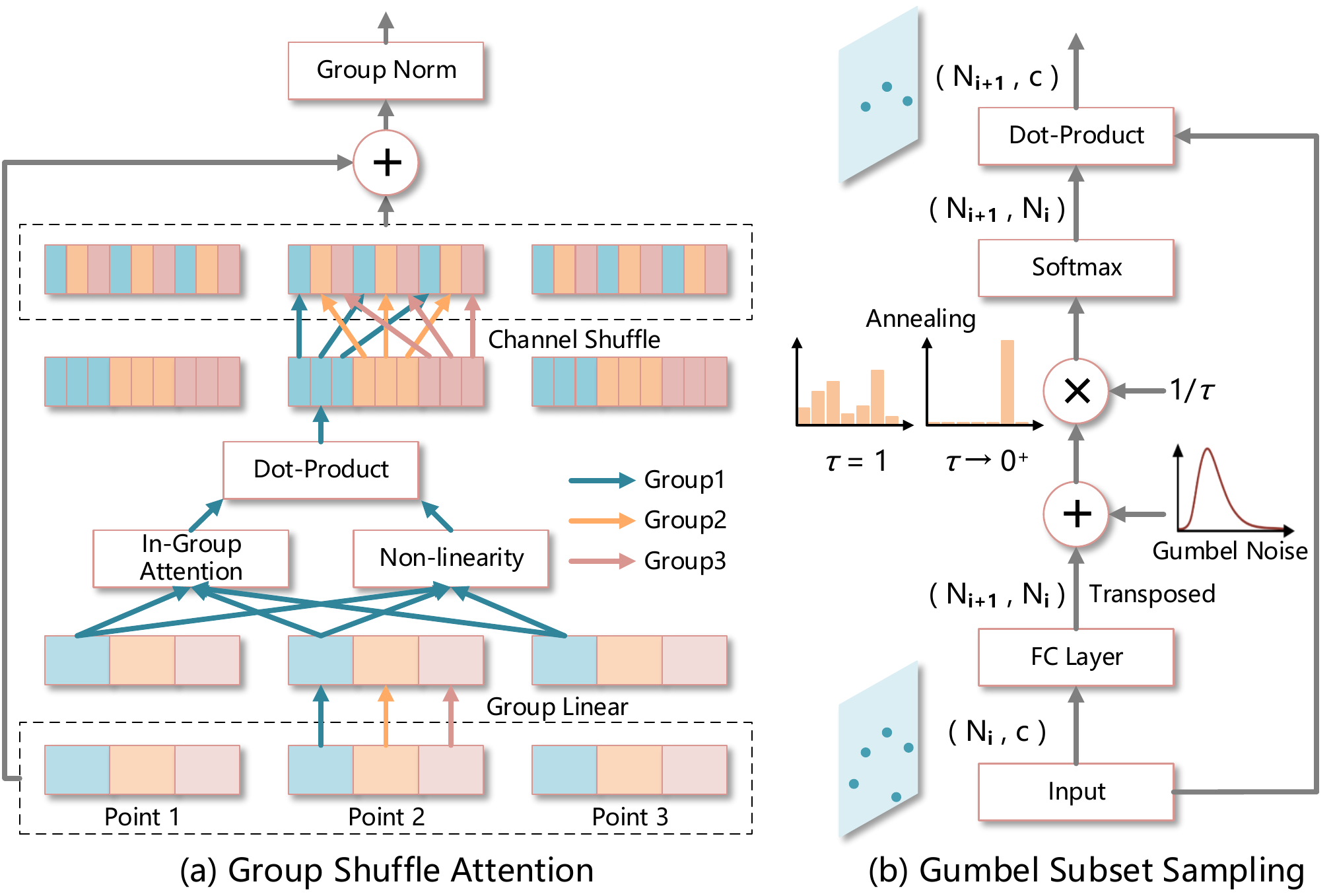}
\caption{\textbf{(a) Group Shuffle Attention.} The core representation learning block in PATs. An input of 3 points is demonstrated. The feature of each point is divided into several groups (different colors in the figure). Within each group, the input points are transformed by a shared linear layer and a non-linear self-attention (Eq. \ref{eq:non-linear-attn}). We then apply channel shuffle on each point feature. Residual connection and group normalization are used for better optimization property. \textbf{(b) Gumbel Subset Sampling} is used for end-to-end learning a representative subset of the input set. For $N_{i+1}$ rounds, one point is selected from the $N_i$ input points competitively. During each round, every input point produces one selection score (Eq. \ref{eq:gumbel-select-one}), with a max score to be selected. A Gumbel-Softmax (Eq. \ref{eq-gumbel-softmax}) with high temperature is used for soft selection in training phase. At inference, a Gumbel-Max (Eq. \ref{eq-gumbel-argmax}) is applied with annealing. Best viewed in color.}
\label{fig:gsa-gss}
\end{figure}

\subsection{Gumbel Subset Sampling} \label{sec:gss}

Although Furthest Point Sampling (FPS) is widely used in point cloud reasoning, it has several defects: 1) its sampling result is dependent on the initial point, \ie, it is not permutation-invariant; 2) it is task-dependent (designed for low-dimension Euclidean space); 3) it is sensitive to outliers. To overcome these issues, we argue that an ideal sampling operation should be:
\begin{itemize}
    \item {\em Permutation-invariant}: the selected subset is always consistent regardless of any permutation of inputs;
    \item {\em Sampling from a high-dimension embedding space}: the sampling operation should be designed {\em task-agnostic} and {\em less sensitive to outliers} by learning representative and robust embeddings; 
    \item {\em Differentiable}: it enables the sampling operation to integrate into neural networks painlessly.
\end{itemize}

For these purposes, we develop a permutation-invariant, task-agnostic and differentiable {\em Gumbel Subset Sampling (GSS)}. Given an input set $X_i \in \mathbb{R}^{N_i\times c}$, which could be output of a neural network layer, the goal is to select a representative subset $X_{i+1} \in \mathbb{R}^{N_{i+1}\times c} \subseteq X_{i}$ with differentiable operations. Inspired by Attention-based MIL pooling \cite{ilse2018attention}, where the pooling output $y$ is an average value weighted by normalized scores produced element-wisely, \ie,
\begin{equation} \label{eq:attn-mil}
    y = \mathit{softmax}(wX_i^T)\cdot X_i, \quad w\in \mathbb{R}^c.
\end{equation}
Note $w\in \mathbb{R}^c$ is a learnable weight and could be replaced with an MLP.

We {\em reinterpret} Attention-based MIL pooling (Eq. \ref{eq:attn-mil}) as competitively selecting one soft virtual point. Though differentiable, the virtual point is however untraceable and less interpretable, especially when selecting multiple points. Instead, we use a hard and discrete selection with an end-to-end trainable \textit{gumbel\_softmax} (Eq. \ref{eq-gumbel-softmax}):
\begin{equation} \label{eq:gumbel-select-one}
    y_{gumbel} = \mathit{gumbel\_softmax}(wX_i^T)\cdot X_i, \quad w\in \mathbb{R}^c.
\end{equation}
in training phase, it provides smooth gradients using discrete reparameterization trick. With annealing, it degenerates to a hard selection in test phase.

A Gumbel Subset Sampling (\textit{GSS}) is simply a multiple-point version of Eq. \ref{eq:gumbel-select-one}, which means a distribution of subsets,
\begin{equation} \label{eq:gss}
    \mathit{GSS}(X_i) = \mathit{gumbel\_softmax}(WX_i^T)\cdot X_i, \quad W\in \mathbb{R}^{N_{i+1}\times c}.
\end{equation}

The following proposition theoretically guarantees the permutation-invariance of \textit{GSS}.
\begin{proposition}
    The Gumbel Subset Sampling operation is permutation-invariant, \ie, given input $X \in \mathbb{R}^{N\times c}$, $\forall$ permutation matrix $P$ of size $N$,
    \begin{equation*}
        \mathit{GSS}(P\cdot X) = \mathit{GSS}(X).
    \end{equation*}
\end{proposition}
Proof is provided in Appendix \ref{sec:proof2}.

\subsection{Other Architecture Design}

\paragraph{Down-sampling Structure} \label{sec:down-sample}
In our classification models, we down-sample input points at 3 levels (from 1,024 points to 384 - 128 - 64 points). Although \textit{GSS} is theoretically superior to FPS, the Gumbel noises also serve as a (too) strong regularization.
Instead of using \textit{GSS} in all down-sampling, we find that replacing the first down-sampling with FPS performs slightly better in our experiments.

\paragraph{Element-wise Loss} \label{sec:elem-loss}
We compute the classification loss as segmentation \cite{pointcnn}: a shared MLP is connected to each remaining point to output the same target class, where the MLP is a stack of "FC - $\mathcal{GN}$ - ELU - dropout \cite{srivastava2014dropout}". The final loss is averaged by element-wise cross entropy. The element-wise loss trick does not bring any performance boost, while the training is significantly faster to converge. At inference, the final classification score is averaged by the element-wise outputs. 





\section{Applications}
In this section, we first demonstrate the effectiveness and efficiency of PATs on a benchmark of point cloud classification, ModelNet40 dataset \cite{wu20153d} of CAD models. We then explore the model performance on real-world datasets. We report the segmentation results on S3DIS dataset \cite{armeni20163d}. Furthermore, we propose a novel application on recognizing gestures with event camera on DVS128 Gesture Dataset \cite{amir2017low}. To our knowledge, this is the first study to process event-camera stream as spatio-temporal point clouds, with {\em state-of-the-art} performance.

\subsection{ModelNet40 Shape Classification} \label{sec:modelnet-results}
\paragraph{Dataset} 
We evaluate our classification model on ModelNet40 \cite{wu20153d} dataset of 40-category CAD models. Official split with 9,840 samples for training, and 2,468 samples for test is used in our experiments. We use the same preprocessed dataset as PointNet++ \cite{pointnetplusplus}.

\begin{table}
\centering

\begin{tabular}{lcc}
\toprule
Method          & Points    & Accuracy (\%)  \\ \midrule
DeepSets \cite{zaheer2017deep}          & 5,000     & 90.0    \\
PointNet   \cite{qi2016pointnet}        & 1,024     & 89.2    \\ 
Kd-Net \cite{klokov2017escape}      & 1,024     & 90.6   \\ 
PointNet++ \cite{pointnetplusplus}      & 1,024     & 90.7    \\ 
KCNet \cite{shen2018mining}      & 1,024     & 91.0    \\ 
DGCNN \cite{wang2018dynamic} & 1,024 & {\bf 92.2} \\
PointCNN \cite{pointcnn}      & 1,024     & {\bf 92.2}   \\ \midrule
PAT (\textit{GSA only})      & 1,024     & 91.3       \\ 
PAT (\textit{GSA only})       & 256       & 90.9      \\ 
PAT (FPS)             & 1,024     & 91.4  \\ 
PAT (FPS + \textit{GSS})         & 1,024     & 91.7   \\ 
\bottomrule
\end{tabular}
\caption{Classification performance on ModelNet40 dataset. }
\label{table:modelnet40}
\end{table}

\paragraph{Experiment Setting} \label{sec:modelnet40-setting}
Classification PATs use \textit{ARPE} to produce 1,024-dimension embeddings, subsequently fed into 3 \textit{GSAs} with hidden size 1,024, followed by a shared MLP with 1,024 - 512 - 256 hidden sizes for 40-category element-wise cross entropy loss (Section \ref{sec:elem-loss}). Several variants of PATs are considered in our experiments: "PAT (\textit{GSA only})" uses no down-sampling; "PAT (FPS)" uses FPS down-sampling after each \textit{GSA}, with a FPS(384) - FPS(128) - FPS(64) down-sampling structure; and "PAT (FPS + \textit{GSS})" uses a down-sampling structure \textit{GSS} except for the first one, \ie FPS(384) - \textit{GSS}(128) - \textit{GSS}(64). 

We train the neural networks using Adam optimizer \cite{kingma2014adam}, with a batch size of 64 and initial learning rate of 0.001. We halve the learning rate every 15 epochs, and 150 epochs are enough for convergence. 

\paragraph{Performance and Model Complexity} 

Classification performance on the test set is summarized in Table \ref{table:modelnet40} with recent {\em state-of-the-art}. Our PATs (including all variants) achieve comparable result on ModelNet40. Interestingly, the PAT using only 256 points (to train and test) outperforms the models before PointNet++ \cite{pointnetplusplus} using 1,024 points.

We also evaluate the model complexity in terms of model size and forward time in Table \ref{table:modelnet40-complexity}. The forward time is recorded with a batch size of 8 on a single GTX 1080 GPU, which is the same hardware environment of the comparison models. As illustrated, our models achieve competitive performance with great parameter-efficiency and acceptable speed. Due to the insufficient support of group linear layers in PyTorch (0.4.1) \cite{paszke2017automatic}, there still exists improvements in speed with low-level implemental optimization. Note the PATs with down-sampling achieve better performance with even lower computation cost, and \textit{GSS} improves FPS further with a neglectable burden.

\begin{table}
\centering

\begin{tabular}{lccc}
\toprule
Method          & Size  & Time & Accuracy (\%)\\ \midrule

PointNet   \cite{qi2016pointnet}        & 40     & {\bf 25.3}& 89.2    \\ 
 
PointNet++ \cite{pointnetplusplus}      & 12     & 163.2 & 90.7    \\ 

DGCNN \cite{wang2018dynamic} & 21 & 94.6 &{\bf 92.2} \\ \midrule

PAT (\textit{GSA only})      & {\bf 5}  & 132.9   & 91.3  \\ 
PAT (FPS)             & {\bf 5}    & 87.6  & 91.4 \\ 
PAT (FPS + \textit{GSS})         & 5.8 & 88.6 & 91.7 \\ 
\bottomrule
\end{tabular}
\caption{Model size ("Size", MB), forward time ("Time", ms) and Accuracy on ModelNet40 dataset.}
\label{table:modelnet40-complexity}
\end{table}

\subsection{S3DIS Indoor Scene Segmentation} \label{sec:s3dis}
\paragraph{Dataset}
We evaluate our PAT segmentation models on real-word point cloud semantic segmentation dataset, Stanford Large-Scale 3D Indoor Spaces Dataset (S3DIS) \cite{armeni20163d}. This dataset contains 3D RGB point clouds of 6 indoor areas totally including 272 rooms. Each point belongs to one of 13 semantic categories (\eg, ceiling, floor, clutter). 

\paragraph{Experiment Setting}
We follow the same setting as prior study \cite{pointcnn}, where each room is split into blocks of area 1.5\textit{m} $\times$ 1.5\textit{m}, and each point is represented as a 6D vector (XYZ, RGB). 2,048 points are sampled for each block during training process, and all points are used for testing block-wisely. We use a 6-fold cross validation over the 6 areas, with 5 areas for training and 1 area left for validation each time. As there are overlaps between areas except for Area 5 \cite{SPGraph}, we report the metrics on Area 5 separately. 

Segmentation PATs use \textit{ARPE} modules to produce 1,024-dimension embeddings, followed by 5 1,024-dimension \textit{GSAs}. No down-sampling is used. A shared MLP with the same structure as that in our classification PATs (Section \ref{sec:modelnet40-setting}) is used for 13-category segmentation. Adam optimizer \cite{kingma2014adam} is used for training cross-entropy loss with a batch size of 16. The learning rate is initialized at 0.0001, then halved every 5 epochs. The training is converged within 20 epochs.



\paragraph{Performance}
Evaluation performance on all-area cross validation (AREAS) and Area 5 is reported in Table \ref{table:S3DIS}. Our segmentation PAT achieves a best trade-off between segmentation performance and parameter-efficiency. On Area 5, it outperforms all the comparison models; on AREAS, our method achieves a superior performance over all comparison models except for PointCNN \cite{pointcnn} in terms of mIoU, with a significantly smaller model size.

To further analyze the performance between PointCNN and our method, we compare per-class IoU and mean per-class accuracy (mAcc) on AREAS and Area 5. As depicted in Table \ref{table:pointcnn-pat}, on AREAS, our method outperforms PointCNN in terms of mAcc; on Area 5, our method outperforms PointCNN in terms of both mIoU and mAcc, plus superior per-class IoUs on majority of classes.

\begin{table}
\centering
\begin{tabular}{lcccc}
\toprule
Method  & mIoU  & mIoU on Area 5   & Size (MB) \\ \midrule
RSNet \cite{RSNet}  & 56.47  & -        & -            \\
SPGraph \cite{SPGraph} & 62.1    & 58.04 & -                 \\
PointNet \cite{qi2016pointnet}          & 47.71    & 47.6     & {\bf 4.7}          \\ 
DGCNN \cite{wang2018dynamic}   & 56.1    & -   & 6.9                 \\ 
PointCNN \cite{pointcnn}   & {\bf 65.39}    & 57.26     & 46.2           \\ \midrule
PAT  & 64.28        & {\bf 60.07}    & 6.1               \\ 
\bottomrule
\end{tabular}
\caption{3D semantic segmentation results on S3DIS. Mean per-class IoU (mIoU, \%) is used as evaluation metric. Model sizes are obtained using the official codes.}
\label{table:S3DIS}
\end{table}

\begin{table*}[]
\centering
{\fontsize{9pt}{9pt}\selectfont
\setlength{\tabcolsep}{1.0mm}{
\begin{tabular}{c|c|c|c|c|c|c|c|c|c|c|c|c|c|c|c|c}
\hline

\multicolumn{2}{c|}{Class}    & ceiling & floor & wall  & beam  & colum & window & door  & table & chair & sofa  & bookcase & board & clutter & mIoU  & mAcc  \\ \hline

\multirow{2}{*}{AREAS} & PointCNN & 94.78   & 97.30 & 75.82 & 63.25 & 51.71  & 58.38  & 57.18 & 71.63 & 69.12 & 39.08 & 61.15    & 52.19 & 58.59   & 65.39 & 75.61  \\ 
                           & PAT    & 93.01   & \textbf{98.36} & 73.54 & 58.51 & 38.87  & \textbf{77.41}  & \textbf{67.74} & 62.70 & 67.30 & 30.63 & 59.60    & \textbf{66.61} & 41.39 & 64.28   & \textbf{76.45} \\ \hline

\multirow{2}{*}{Area 5}     & PointCNN & 92.31   & 98.24 & 79.41 & 0.00     & 17.6   & 22.77  & 62.09 & 74.39 & 80.59 & 31.67 & 66.67    & 62.05 & 56.74  & 57.26 & 63.86  \\ 
                           & PAT   & \textbf{93.04}  & \textbf{98.51} & 72.28 & \textbf{1.00}  & \textbf{41.52}  & \textbf{85.05}  & 38.22 & 57.66 & \textbf{83.64} & \textbf{48.12} & \textbf{67.00}    & 61.28 & 33.64  & \textbf{60.07} & \textbf{70.83}  \\ \hline

\end{tabular}}
}
\caption{Comparison between PointCNN \cite{pointcnn} and our PAT on all-area cross validation (AREAS) and Area 5 on S3DIS dataset, in terms of per-class IoU (\%), mean per-class IoU (mIoU, \%), and mean per-class accuracy (mAcc, \%)}
\label{table:pointcnn-pat}
\end{table*}

\subsection{Event Camera Stream as Point Clouds: DVS128 Gesture Recognization}

\paragraph{Motivation and Dataset}
Point cloud approaches are primarily designed for 3D spatial sensors, \eg, LiDAR and Matterport 3D Cameras. However, there are numbers of potential applications with point-based records. In this section, we explore a novel application on event camera with point cloud approaches.

Dynamic Vision Sensor (DVS) \cite{lichtsteiner2008128} is a biologically inspired event camera, which "transmits data only when a pixel detects a change" \cite{amir2017low}. On the 128$\times$128 sensor matrix, it records whether there is a change (by a user-defined threshold) on the corresponding position in microseconds. In particular, we explore gesture recognition on DVS128 Gesture Dataset \cite{amir2017low}, with 11 classes of gestures (10 named gestures, \eg, "Arm Roll", and 1 "others") collected from 122 users. Training data is collected from 94 users with 939 event streams, and test data is collected from 24 users with 249 event streams. The gesture action records is a sequence of points, each of which is {\em point change} represented as a 4-dimension vector: abscissa $x$, ordinate $y$, timestamp $t$, and polarity (1 for appear and 0 for disappear).  In this way, we regard the event stream as {\em spatio-temporal point clouds}. See Figure \ref{fig:dvs128} for illustration.

\begin{figure}[t]
\centering
\includegraphics[width=\linewidth]{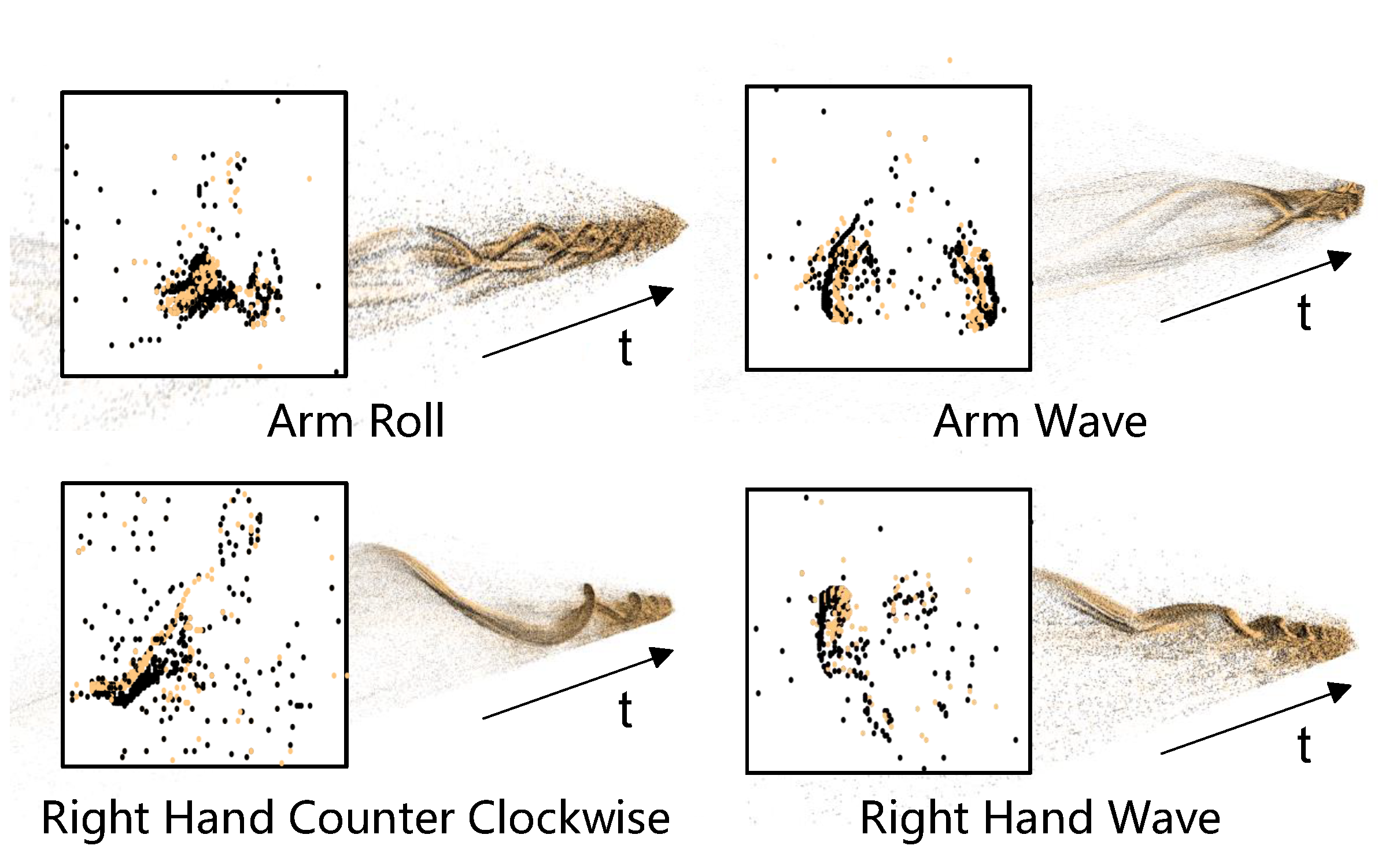}
\caption{\textbf{Visualization of DVS128 Gesture Dataset}. 4 event streams are displayed as spatio-temporal point clouds. The patch on the left of each stream is a snapshot of one timestep. Points are colored in black for polarity 1 or yellow for polarity 0.}
 
\label{fig:dvs128}
\end{figure}

\paragraph{Experiment Setting}
We use a sliding window approach to get training and test samples. Each sample is within a {\em window length} of 750ms, sliding with {\em step size} of 100ms, see Figure \ref{fig:preprocess} for demonstration. After this preprocessing, there are 61,544 clips for training and 10,256 clips for test. The step size could also be regarded as the maximum decision latency for a realtime application. For a stream containing $m$ clips, its {\em system-level} prediction is the mode of all predictions labels. The System Accuracy based on the system-level prediction is used for evaluation \cite{amir2017low}.

  
  
The same-structured classification PATs as ModelNet40 (Section \ref{sec:modelnet40-setting}) are used in this experiment. An SGD optimizer with a constant learning rate of 0.001 is used to train the PATs within 60 epochs for convergence.  
 
  
\begin{figure}[ht]
\centering
\includegraphics[width=\linewidth]{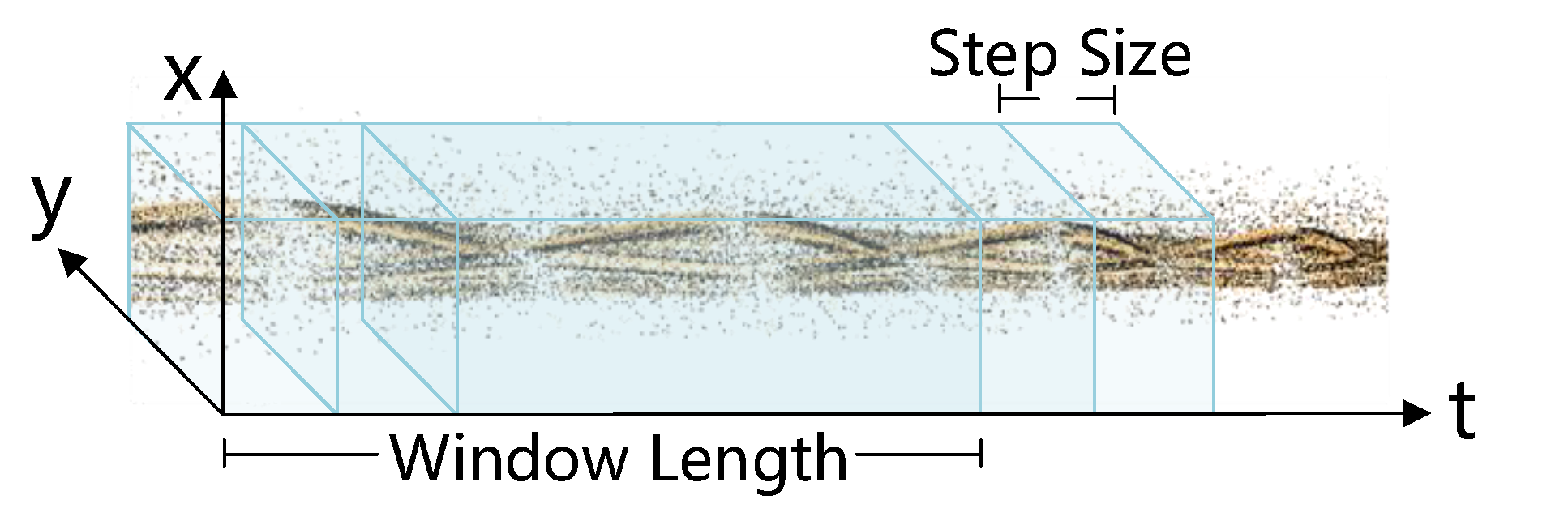}
\caption{Window length and step size in event camera stream.}
\label{fig:preprocess}
\end{figure}
  
  
\paragraph{Performance} There are few studies using point cloud approaches on event camera stream. For fair comparison, we implement a vanilla PointNet \cite{qi2016pointnet} and a PointNet++ \cite{pointnetplusplus} on this experiment. If not specified, 1,024 points sampled from the clips are used for training and evaluation.

As depicted in Table \ref{table:dvsresults}, all point cloud approaches are running within the maximum decision latency (100ms). We achieve a {\em state-of-the-art} on this dataset, with strong parameter-efficiency. Interestingly, PAT (\textit{GSA only}) with 256 points performs similarly to that with 1,024 points. We argue that it is because of outliers and sparsity of the events. Note the baseline CNN \cite{amir2017low} is running on a low-power processor with a maximum decision latency of 105ms. Our results indicate the potential of replacing CNNs with PATs, with general network quantization techniques \cite{NIPS2016_6573}.

\begin{table} 
\centering
{\fontsize{10pt}{10pt}\selectfont
\setlength{\tabcolsep}{1.0mm}{
\begin{tabular}{lccccc}
\hline
Model & 10-\textsc{class}  & 11-\textsc{class}  & Time& Size \\ \hline
CNN \cite{amir2017low}          &  96.5 &  94.4 & -          & -\\ \hline 
PointNet  \cite{qi2016pointnet}   & 89.5 & 88.8 & 2.8          & 6.5    \\
PointNet++ \cite{pointnetplusplus}  & 95.6 & 95.2 & 18.2         & 12         \\ \hline
PAT  (\textit{GSA only})   &  96.9 &  95.6 & 16.9    & 5          \\
PAT (\textit{GSA only}, N256)      & 96.9 &95.6 & 7.5    & 5          \\
PAT (FPS)      &  96.5 & 95.2 & 12.7         & 5         \\
\text{PAT (FPS + \textit{GSS})}    & {\bf 97.4} & {\bf 96.0} & 13.1    & 5.8        \\ \hline

\end{tabular}}
}
\caption{DVS128 Gesture Dataset System Accuracy in 10 classes ("10-\textsc{class}") and 11 classes ("11-\textsc{class}"). "N256" means using 256 points to train and test. Forward time ("Time") with a batch size of 1 on a single GTX 1080 GPU is measured in ms, and model size ("Size") is measured in MB.}
\label{table:dvsresults}
\end{table}

\section{Ablation Study}

In this section, we analyze the effectiveness of the components / tricks on ModelNet40 dataset. As \textit{GSS} has been proven effective in Section \ref{sec:modelnet-results}, we analyze the components on PATs without down-sampling. All experiments in the ablation study are conducted using 256 points.
 
\begin{figure}[ht]
\centering
\includegraphics[width=\linewidth]{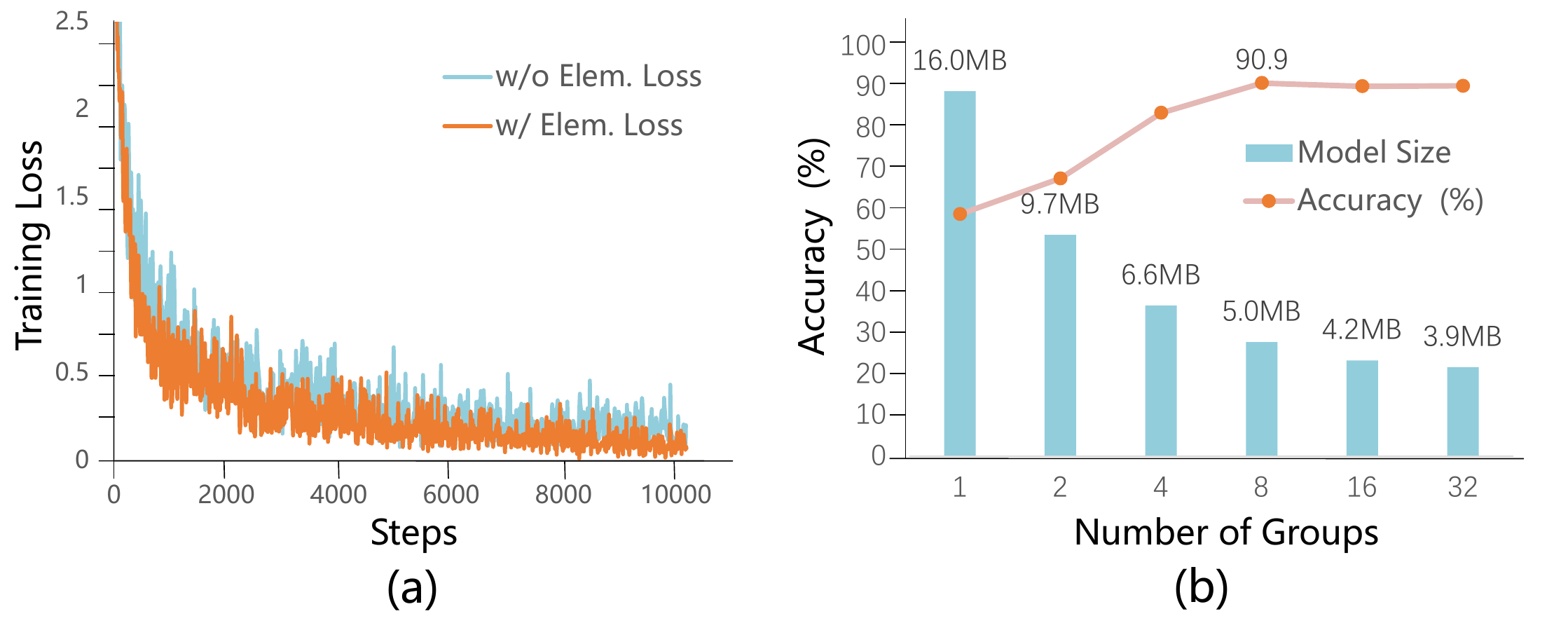}
\caption{\textbf{(a) ModelNet40 training loss} with (w/) or without (w/o) element-wise loss. \textbf{(b) ModelNet40 accuracy (with 256 points) and model size} by varying numbers of groups.}
\label{fig:loss_acc_size}
\end{figure}

{\bf Element-wise Loss.} As depicted in Figure \ref{fig:loss_acc_size} (a), training tends to be faster with element-wise loss. However, there is no performance boost for evaluation on test set. 

{\bf Number of Groups.} As shown in Figure \ref{fig:loss_acc_size} (b), grouping is critical to the performance since it is coupled with the multi-head design in attention. Without grouping ($group=1$), the model accuracy drops significantly, with even larger model size. With 8 groups, it achieves a best trade-off between accuracy and model size.

{\bf Channel Shuffle vs. no Shuffle.} To enable information to flow across groups, channel shuffle is critical to \textit{GSA} (CS "On" or "Off" in Table \ref{table:component-accuracy}), which is parameter-free and introduces neglectable computation cost.

{\bf Embedding Layer.} \textit{ARPE} module is shown to be very effective to boost performance ("MLP" or "\textit{ARPE}" in Table \ref{table:component-accuracy}). It provides an improvement of approximately 0.8\% consistently on 256 or 1,024 points.

{\bf \textit{GSA} vs. \textit{MHA}.} We design 2 \textit{MHA} counterparts to compare with: 1) \textit{MHA LG}, with the same hidden size as \textit{GSA}, and 2) \textit{MHA SM}, by tuning the hidden size to keep a comparable model size as \textit{GSA}. As depicted in Table \ref{table:component-accuracy}, our PATs with \textit{GSA} show superior performance in terms of both parameter efficiency and accuracy.

{\bf Group Norm vs. Layer Norm.}
We also discuss layer normalization $\mathcal{LN}$ \cite{ba2016layer} in the original Multi-head Attention \cite{vaswani2017attention}. As $\mathcal{GN}$ is proposed to be an extension to $\mathcal{LN}$, in our experiments ($\mathcal{GN}$ or $\mathcal{LN}$ in Table \ref{table:component-accuracy}), the former ($group=8$) outperforms the latter consistently.

\begin{table}
\centering

\begin{tabular}{ccccc}
\toprule
Embedding & Attention & CS  & Norm          & Accuracy  \\ \midrule
\textit{ARPE}  & \textit{GSA}& On       &$\mathcal{GN}$        & 90.9\\
\textit{ARPE}  & \textit{GSA}& Off       &$\mathcal{GN}$        & 88.8 \\
MLP  & \textit{GSA}& On       &$\mathcal{GN}$        & 90.1 \\
\textit{ARPE}  & \textit{MHA SM}& On       &$\mathcal{GN}$        & 89.3\\
\textit{ARPE}  & \textit{MHA LG}& On       &$\mathcal{GN}$        & 89.9\\
\textit{ARPE}  & \textit{GSA}& On       &$\mathcal{LN}$        & 89.9\\
MLP  & \textit{GSA}& On       &$\mathcal{LN}$        & 90.0\\

\bottomrule
\end{tabular}
\caption{Effectiveness of components in PATs. CS denotes channel shuffle. Accuracy is obtained using 256 points on ModelNet40.}
\label{table:component-accuracy}
\end{table}
 
\section{Conclusion}
We develop {\em Point Attention Transformers (PATs)} on point cloud reasoning. A parameter-efficient {\em Group Shuffle Attention (GSA)} is proposed to learn the relations between points. Besides, we design an end-to-end learnable and task-agnostic sampling operation, named {\em Gumbel Subset Sampling (GSS)}. Results on several benchmarks demonstrate the effectiveness and efficiency of our methods.
In the future, it is interesting to apply \textit{GSS} on general sets, \eg, to explore both effectiveness and interpretability on hierarchical multiple instance learning.

\noindent\textbf{Acknowledgment}
\small 
This work was supported by National Science Foundation of China (U1611461, 61521062). This work was partly supported by STCSM (18DZ1112300, 18DZ2270700). This work was also partially supported by joint research grant of SJTU-BIGO LIVE, joint research grant of SJTU-Minivision, and China's Thousand Talent Program. 
 
\small
\bibliography{ref}

\begin{thebibliography}{10}\itemsep=-1pt

\bibitem{amir2017low}
Arnon Amir, Brian Taba, David~J Berg, Timothy Melano, Jeffrey~L McKinstry,
  Carmelo Di~Nolfo, Tapan~K Nayak, Alexander Andreopoulos, Guillaume Garreau,
  Marcela Mendoza, et~al.
\newblock A low power, fully event-based gesture recognition system.
\newblock In {\em CVPR}, pages 7388--7397, 2017.

\bibitem{armeni20163d}
Iro Armeni, Ozan Sener, Amir~R Zamir, Helen Jiang, Ioannis Brilakis, Martin
  Fischer, and Silvio Savarese.
\newblock 3d semantic parsing of large-scale indoor spaces.
\newblock In {\em CVPR}, pages 1534--1543, 2016.

\bibitem{ba2016layer}
Jimmy~Lei Ba, Jamie~Ryan Kiros, and Geoffrey~E Hinton.
\newblock Layer normalization.
\newblock {\em arXiv preprint arXiv:1607.06450}, 2016.

\bibitem{bahdanau2014neural}
Dzmitry Bahdanau, Kyunghyun Cho, and Yoshua Bengio.
\newblock Neural machine translation by jointly learning to align and
  translate.
\newblock In {\em ICLR}, 2015.

\bibitem{bengio2013estimating}
Yoshua Bengio, Nicholas L{\'e}onard, and Aaron Courville.
\newblock Estimating or propagating gradients through stochastic neurons for
  conditional computation.
\newblock {\em arXiv preprint arXiv:1308.3432}, 2013.

\bibitem{bronstein2017geometric}
Michael~M Bronstein, Joan Bruna, Yann LeCun, Arthur Szlam, and Pierre
  Vandergheynst.
\newblock Geometric deep learning: going beyond euclidean data.
\newblock {\em IEEE Signal Processing Magazine}, 34(4):18--42, 2017.

\bibitem{chollet2017xception}
Fran{\c{c}}ois Chollet.
\newblock Xception: Deep learning with depthwise separable convolutions.
\newblock In {\em CVPR}, pages 1251--1258, 2017.

\bibitem{clevert2015fast}
Djork-Arn{\'e} Clevert, Thomas Unterthiner, and Sepp Hochreiter.
\newblock Fast and accurate deep network learning by exponential linear units
  (elus).
\newblock In {\em ICLR}, 2016.

\bibitem{cohen2018spherical}
Taco~S Cohen, Mario Geiger, Jonas K{\"o}hler, and Max Welling.
\newblock Spherical cnns.
\newblock In {\em ICLR}, 2018.

\bibitem{devlin2018bert}
Jacob Devlin, Ming-Wei Chang, Kenton Lee, and Kristina Toutanova.
\newblock Bert: Pre-training of deep bidirectional transformers for language
  understanding.
\newblock {\em arXiv preprint arXiv:1810.04805}, 2018.

\bibitem{gehring2017convolutional}
Jonas Gehring, Michael Auli, David Grangier, Denis Yarats, and Yann Dauphin.
\newblock Convolutional sequence to sequence learning.
\newblock In {\em ICML}, 2017.

\bibitem{he2016deep}
Kaiming He, Xiangyu Zhang, Shaoqing Ren, and Jian Sun.
\newblock Deep residual learning for image recognition.
\newblock In {\em CVPR}, pages 770--778, 2016.

\bibitem{RSNet}
Qiangui Huang, Weiyue Wang, and Ulrich Neumann.
\newblock Recurrent slice networks for 3d segmentation of point clouds.
\newblock In {\em CVPR}, pages 2626--2635, 2018.

\bibitem{NIPS2016_6573}
Itay Hubara, Matthieu Courbariaux, Daniel Soudry, Ran El-Yaniv, and Yoshua
  Bengio.
\newblock Binarized neural networks.
\newblock In {\em NIPS}, pages 4107--4115, 2016.

\bibitem{ilse2018attention}
Maximilian Ilse, Jakub~M Tomczak, and Max Welling.
\newblock Attention-based deep multiple instance learning.
\newblock In {\em ICML}, 2018.

\bibitem{jang2016categorical}
Eric Jang, Shixiang Gu, and Ben Poole.
\newblock Categorical reparameterization with gumbel-softmax.
\newblock In {\em ICLR}, 2017.

\bibitem{kingma2014adam}
Diederik~P Kingma and Jimmy Ba.
\newblock Adam: A method for stochastic optimization.
\newblock In {\em ICLR}, 2015.

\bibitem{kingma2013auto}
Diederik~P Kingma and Max Welling.
\newblock Auto-encoding variational bayes.
\newblock In {\em ICLR}, 2014.

\bibitem{klokov2017escape}
Roman Klokov and Victor Lempitsky.
\newblock Escape from cells: Deep kd-networks for the recognition of 3d point
  cloud models.
\newblock In {\em ICCV}, pages 863--872, 2017.

\bibitem{SPGraph}
Lo{\"i}c Landrieu and Martin Simonovsky.
\newblock Large-scale point cloud semantic segmentation with superpoint graphs.
\newblock In {\em CVPR}, pages 4558--4567, 2018.

\bibitem{le2018pointgrid}
Truc Le and Ye Duan.
\newblock Pointgrid: A deep network for 3d shape understanding.
\newblock In {\em CVPR}, pages 9204--9214, 2018.

\bibitem{li2018so}
Jiaxin Li, Ben~M Chen, and Gim~Hee Lee.
\newblock So-net: Self-organizing network for point cloud analysis.
\newblock In {\em CVPR}, pages 9397--9406, 2018.

\bibitem{pointcnn}
Yangyan Li, Rui Bu, Mingchao Sun, Wei Wu, Xinhan Di, and Baoquan Chen.
\newblock Pointcnn: Convolution on x-transformed points.
\newblock In {\em NeurIPS}, 2018.

\bibitem{lichtsteiner2008128}
Patrick Lichtsteiner, Christoph Posch, and Tobi Delbruck.
\newblock A 128x128 120 db 15 mu s latency asynchronous temporal contrast
  vision sensor.
\newblock {\em IEEE journal of solid-state circuits}, 43(2):566--576, 2008.

\bibitem{luong2015effective}
Minh-Thang Luong, Hieu Pham, and Christopher~D Manning.
\newblock Effective approaches to attention-based neural machine translation.
\newblock In {\em EMNLP}, 2015.

\bibitem{maddison2016concrete}
Chris~J Maddison, Andriy Mnih, and Yee~Whye Teh.
\newblock The concrete distribution: A continuous relaxation of discrete random
  variables.
\newblock In {\em ICLR}, 2017.

\bibitem{mena2018learning}
Gonzalo Mena, David Belanger, Scott Linderman, and Jasper Snoek.
\newblock Learning latent permutations with gumbel-sinkhorn networks.
\newblock In {\em ICLR}, 2018.

\bibitem{parmar2018image}
Niki Parmar, Ashish Vaswani, Jakob Uszkoreit, {\L}ukasz Kaiser, Noam Shazeer,
  and Alexander Ku.
\newblock Image transformer.
\newblock In {\em ICML}, 2018.

\bibitem{paszke2017automatic}
Adam Paszke, Sam Gross, Soumith Chintala, Gregory Chanan, Edward Yang, Zachary
  DeVito, Zeming Lin, Alban Desmaison, Luca Antiga, and Adam Lerer.
\newblock Automatic differentiation in pytorch.
\newblock In {\em NIPS-W}, 2017.

\bibitem{qi2016pointnet}
Charles~Ruizhongtai Qi, Hao Su, Kaichun Mo, and Leonidas~J. Guibas.
\newblock Pointnet: Deep learning on point sets for 3d classification and
  segmentation.
\newblock In {\em CVPR}, pages 77--85, 2017.

\bibitem{qi2016volumetric}
Charles~R Qi, Hao Su, Matthias Nie{\ss}ner, Angela Dai, Mengyuan Yan, and
  Leonidas~J Guibas.
\newblock Volumetric and multi-view cnns for object classification on 3d data.
\newblock In {\em CVPR}, pages 5648--5656, 2016.

\bibitem{pointnetplusplus}
Charles~Ruizhongtai Qi, Li Yi, Hao Su, and Leonidas~J. Guibas.
\newblock Pointnet++: Deep hierarchical feature learning on point sets in a
  metric space.
\newblock In {\em NIPS}, pages 5105--5114, 2017.

\bibitem{schulman2015gradient}
John Schulman, Nicolas Heess, Theophane Weber, and Pieter Abbeel.
\newblock Gradient estimation using stochastic computation graphs.
\newblock In {\em NIPS}, pages 3528--3536, 2015.

\bibitem{shen2018mining}
Yiru Shen, Chen Feng, Yaoqing Yang, and Dong Tian.
\newblock Mining point cloud local structures by kernel correlation and graph
  pooling.
\newblock In {\em CVPR}, volume~4, 2018.

\bibitem{srivastava2014dropout}
Nitish Srivastava, Geoffrey Hinton, Alex Krizhevsky, Ilya Sutskever, and Ruslan
  Salakhutdinov.
\newblock Dropout: a simple way to prevent neural networks from overfitting.
\newblock {\em The Journal of Machine Learning Research}, 15(1):1929--1958,
  2014.

\bibitem{su2018splatnet}
Hang Su, Varun Jampani, Deqing Sun, Subhransu Maji, Evangelos Kalogerakis,
  Ming-Hsuan Yang, and Jan Kautz.
\newblock Splatnet: Sparse lattice networks for point cloud processing.
\newblock In {\em CVPR}, pages 2530--2539, 2018.

\bibitem{su2015multi}
Hang Su, Subhransu Maji, Evangelos Kalogerakis, and Erik Learned-Miller.
\newblock Multi-view convolutional neural networks for 3d shape recognition.
\newblock In {\em CVPR}, pages 945--953, 2015.

\bibitem{sutton2000policy}
Richard~S Sutton, David~A McAllester, Satinder~P Singh, and Yishay Mansour.
\newblock Policy gradient methods for reinforcement learning with function
  approximation.
\newblock In {\em NIPS}, pages 1057--1063, 2000.

\bibitem{vaswani2017attention}
Ashish Vaswani, Noam Shazeer, Niki Parmar, Jakob Uszkoreit, Llion Jones,
  Aidan~N Gomez, {\L}ukasz Kaiser, and Illia Polosukhin.
\newblock Attention is all you need.
\newblock In {\em NIPS}, pages 6000--6010, 2017.

\bibitem{velickovic2017graph}
Petar Velickovic, Guillem Cucurull, Arantxa Casanova, Adriana Romero, Pietro
  Lio, and Yoshua Bengio.
\newblock Graph attention networks.
\newblock In {\em ICLR}, 2018.

\bibitem{wang2017cnn}
Peng-Shuai Wang, Yang Liu, Yu-Xiao Guo, Chun-Yu Sun, and Xin Tong.
\newblock O-cnn: Octree-based convolutional neural networks for 3d shape
  analysis.
\newblock {\em ACM Transactions on Graphics}, 36(4):72, 2017.

\bibitem{wang2018dynamic}
Yue Wang, Yongbin Sun, Ziwei Liu, Sanjay~E Sarma, Michael~M Bronstein, and
  Justin~M Solomon.
\newblock Dynamic graph cnn for learning on point clouds.
\newblock {\em arXiv preprint arXiv:1801.07829}, 2018.

\bibitem{wu2018group}
Yuxin Wu and Kaiming He.
\newblock Group normalization.
\newblock In {\em ECCV}, 2018.

\bibitem{wu20153d}
Zhirong Wu, Shuran Song, Aditya Khosla, Fisher Yu, Linguang Zhang, Xiaoou Tang,
  and Jianxiong Xiao.
\newblock 3d shapenets: A deep representation for volumetric shapes.
\newblock In {\em CVPR}, pages 1912--1920, 2015.

\bibitem{xie2017aggregated}
Saining Xie, Ross Girshick, Piotr Doll{\'a}r, Zhuowen Tu, and Kaiming He.
\newblock Aggregated residual transformations for deep neural networks.
\newblock In {\em CVPR}, pages 5987--5995, 2017.

\bibitem{xie2018attentional}
Saining Xie, Sainan Liu, Zeyu Chen, and Zhuowen Tu.
\newblock Attentional shapecontextnet for point cloud recognition.
\newblock In {\em CVPR}, pages 4606--4615, 2018.

\bibitem{xu2015show}
Kelvin Xu, Jimmy Ba, Ryan Kiros, Kyunghyun Cho, Aaron Courville, Ruslan
  Salakhudinov, Rich Zemel, and Yoshua Bengio.
\newblock Show, attend and tell: Neural image caption generation with visual
  attention.
\newblock In {\em ICML}, pages 2048--2057, 2015.

\bibitem{DBLP:conf/mm/YanNY17}
Yichao Yan, Bingbing Ni, and Xiaokang Yang.
\newblock Fine-grained recognition via attribute-guided attentive feature
  aggregation.
\newblock In {\em ACM MM}, pages 1032--1040, 2017.

\bibitem{DBLP:conf/ijcai/YanNY17}
Yichao Yan, Bingbing Ni, and Xiaokang Yang.
\newblock Predicting human interaction via relative attention model.
\newblock In {\em IJCAI}, pages 3245--3251, 2017.

\bibitem{yang2019adversarial}
Jiancheng Yang, Qiang Zhang, Rongyao Fang, Bingbing Ni, Jinxian Liu, and Qi
  Tian.
\newblock Adversarial attack and defense on point sets.
\newblock {\em arXiv preprint arXiv:1902.10899}, 2019.

\bibitem{zaheer2017deep}
Manzil Zaheer, Satwik Kottur, Siamak Ravanbakhsh, Barnabas Poczos, Ruslan~R
  Salakhutdinov, and Alexander~J Smola.
\newblock Deep sets.
\newblock In {\em NIPS}, pages 3391--3401, 2017.

\bibitem{zhang2017interleaved}
Ting Zhang, Guo-Jun Qi, Bin Xiao, and Jingdong Wang.
\newblock Interleaved group convolutions.
\newblock In {\em CVPR}, 2017.

\bibitem{DBLP:journals/corr/ZhangZLS17}
Xiangyu Zhang, Xinyu Zhou, Mengxiao Lin, and Jian Sun.
\newblock Shufflenet: An extremely efficient convolutional neural network for
  mobile devices.
\newblock In {\em CVPR}, pages 6848--6856, 2018.

\bibitem{zhao20183d}
Wei Zhao, Jiancheng Yang, Yingli Sun, Cheng Li, Weilan Wu, Liang Jin, Zhiming
  Yang, Bingbing Ni, Pan Gao, Peijun Wang, et~al.
\newblock 3d deep learning from ct scans predicts tumor invasiveness of
  subcentimeter pulmonary adenocarcinomas.
\newblock {\em Cancer research}, 78(24):6881--6889, 2018.

\bibitem{zhou2017voxelnet}
Yin Zhou and Oncel Tuzel.
\newblock Voxelnet: End-to-end learning for point cloud based 3d object
  detection.
\newblock In {\em CVPR}, 2017.

\end{thebibliography}
\bibliographystyle{ieeefullname}

\clearpage
\appendix
\section*{Appendix}

\newtheorem{lemma}{Lemma}

\section{Proof of Permutation Equivariance of Group Shuffle Attention} \label{sec:proof1}

\begin{lemma}[Permutation matrix and permutation function]
	$\forall \Lambda \in \mathbb{R}^{N \times N}$, $\forall$ permutation matrix $P$ of size $N$, $\exists p:\{1,2,...,N\} \rightarrow \{1,2,...,N\}$ is a permutation function:
	\begin{equation} \label{eq:permutation-property}
	\Lambda_{ij}=(P\cdot\Lambda)_{p(i)j}=(\Lambda \cdot P^T)_{i p(j)}=(P\cdot\Lambda\cdot P^T)_{p(i)p(j)}.
	\end{equation}
\end{lemma}

\begin{lemma}  \label{lemma:softmax}
    Given $A \in \mathbb{R}^{N\times N}$, $\forall$ permutation matrix $P$ of size $N$,
    
    \begin{equation}
     \mathit{softmax}(PAP^T) = P\mathit{softmax}(A)P^T.
    \end{equation}
\end{lemma}

\begin{proof}

$\mathit{softmax}(A)_{ij} = \frac{e^{A_{ij}}}{\sum_{n=1}^{N} e^{A_{ij}}}$.
 
Consider a permutation function $p$ for $P$, using Eq. \ref{eq:permutation-property}, we get:
	\begin{align*}
	\begin{split}
	(P \cdot \mathit{softmax}(A) \cdot P^T)_{p(i)p(j)} &= \mathit{softmax}(A)_{ij} \\
	&= \frac{e^{A_{ij}}}{\sum_{n=1}^{N} e^{A_{ij}}} \\
	&= \frac{e^{(PAP^T)_{p(i)p(j)}}}{\sum_{n=1}^{N} e^{(PAP^T)_{p(i)p(j)}}} \\
	&=  \mathit{softmax}(P \cdot A \cdot P^T)_{p(i)p(j)}.
	\end{split}
	\end{align*}
	which implies $P \cdot \mathit{softmax}(A) \cdot P^T = \mathit{softmax}(P \cdot A \cdot P^T)$.
\end{proof}

\begin{lemma}
Non-linear self-attention $Attn_\sigma(Q,X) = S(Q,X)\cdot \sigma(X)$, with $S(Q,X)=\mathit{softmax}(QX^T / \sqrt{c})$ is permutation-equivariant.
\end{lemma}
\begin{proof}
	For a self-attention,  $Q=X$.
	
	$\sigma$ is an element-wise function, thus $\sigma(P\cdot X) = P\cdot \sigma(X)$. 
	\begin{align*}
	\begin{split}
	Attn_\sigma(PX,PX) &= S(PX,PX)\cdot \sigma(PX)\\
	                   &=\mathit{softmax}(P\cdot XX^T\cdot P^T / \sqrt{c})\cdot P \cdot\sigma(X)\\
	                   &=P\cdot \mathit{softmax}(XX^T/ \sqrt{c})\cdot P^T \cdot P \cdot\sigma(X)\\
	                   &=P\cdot \mathit{softmax}(XX^T/ \sqrt{c})\cdot \cdot\sigma(X)\\
	                   &=P\cdot Attn_\sigma(X,X).
	\end{split}
	\end{align*}
	which implies non-linear self-attention is permutation-equivariant.
\end{proof}

\begin{proposition}
    The Group Shuffle Attention operation is permutation-equivariant, \ie, given input $X \in \mathbb{R}^{N\times c}$, $\forall$ permutation matrix $P$ of size $N$,
    \begin{equation*}
        \mathit{GSA}(P\cdot X) = P\cdot \mathit{GSA}(X).
    \end{equation*}
\end{proposition}
\begin{proof}
\begin{equation} \tag{\ref{eq:gsa}}
    \mathit{GSA}(X) = \mathcal{GN}(\psi(\mathit{GroupAttn}(X)) + X),
\end{equation}
where,
\begin{multline} \tag{\ref{eq:group-attn}}
    \mathit{GroupAttn}(X) = \\
    \mathop{concat}\{Attn_\sigma(X_i,X_i) \,|\,X_i=X^{(i)}W_i\}_{i=1,...,g}.
\end{multline}

\textit{GSA} only introduces element-wise operations, which does not change the permutation-equivariance of $Attn_\sigma$.
\end{proof}

\section{Proof of Permutation Invariance of Gumbel Subset Sampling}  \label{sec:proof2}
\begin{proposition}
    The Gumbel Subset Sampling operation is permutation-invariant, \ie, given input $X \in \mathbb{R}^{N\times c}$, $\forall$ permutation matrix $P$ of size $N$,
    \begin{equation*}
        \mathit{GSS}(P\cdot X) = \mathit{GSS}(X).
    \end{equation*}
\end{proposition}
\begin{proof}
According to Eq. \ref{eq:gss},
    \begin{equation*}
        \mathit{GSS}(X) = \mathit{gumbel\_softmax}(WX^T)\cdot X, \quad W\in \mathbb{R}^{N\times c}.
    \end{equation*}
Similar to Lemma \ref{lemma:softmax},
\begin{equation}
    \mathit{softmax}(AP^T) = \mathit{softmax}(A)P^T.
\end{equation}
Since \textit{gumbel\_softmax} add element-wise operations on \textit{softmax}, it does not change the permutation property. In this way, 
\begin{align*}
	\begin{split}
        GSS(P\cdot X)&=\mathit{gumbel\_softmax}(WX^TP^T)\cdot P X\\
        &=\mathit{gumbel\_softmax}(WX^T)P^T\cdot P X\\
        &=\mathit{gumbel\_softmax}(WX^T)\cdot X\\
        &=GSS(X).
	\end{split}
\end{align*}
\end{proof}


\end{document}